\documentclass{article}

\usepackage[T1]{fontenc}
\usepackage[utf8]{inputenc}
\usepackage{microtype}
\usepackage{graphicx}
\usepackage{subcaption}
\usepackage{booktabs}
\usepackage{tabularx}
\usepackage{hyperref}

\usepackage[accepted]{icml2026}
\makeatletter
\renewcommand{\ICML@appearing}{Accepted to the \textit{ICML 2026
Workshop on Hypothesis Testing}, Seoul, South Korea, 2026.
Copyright 2026 by the author(s).}
\makeatother
\usepackage{amsmath}
\usepackage{amssymb}
\usepackage{mathtools}
\usepackage{amsthm}
\usepackage{enumitem}
\usepackage[capitalize,noabbrev]{cleveref}
\usepackage{xcolor}

\theoremstyle{plain}
\newtheorem{lemma}{Lemma}
\newtheorem{corollary}{Corollary}
\theoremstyle{definition}
\newtheorem{definition}{Definition}

\icmltitlerunning{Resolution Diagnostics for Paired LLM Evaluation}

\begin{document}

\twocolumn[
\icmltitle{Resolution Diagnostics for\\Paired LLM Evaluation}

\begin{icmlauthorlist}
\icmlauthor{Anany Kotawala}{prin}
\end{icmlauthorlist}
\icmlaffiliation{prin}{Princeton University, Princeton, NJ, USA}
\icmlcorrespondingauthor{Anany Kotawala}{akotawala@princeton.edu}
\icmlkeywords{hypothesis testing, paired samples, statistical resolution, minimum detectable effect, McNemar, paired bootstrap, LLM evaluation, multiplicity, leaderboards}
\vskip 0.3in
]

\printAffiliationsAndNotice{}

\hypersetup{
  pdftitle={Resolution Diagnostics for Paired LLM Evaluation},
  pdfsubject={ICML 2026 Workshop on Hypothesis Testing},
  pdfkeywords={hypothesis testing, paired samples, minimum detectable effect, McNemar, paired bootstrap, LLM evaluation, leaderboards}
}

\raggedbottom

\setlength{\textfloatsep}{0pt plus 1pt minus 0pt}
\setlength{\floatsep}{0pt plus 1pt minus 0pt}
\setlength{\intextsep}{0pt plus 1pt minus 0pt}
\setlength{\dbltextfloatsep}{0pt plus 1pt minus 0pt}
\setlength{\dblfloatsep}{0pt plus 1pt minus 0pt}
\setlength{\abovecaptionskip}{2pt}
\setlength{\belowcaptionskip}{2pt}

\begin{abstract}
Across two public LLM leaderboards, many displayed pairwise rankings do not meet a conventional paired-test resolution target under the actual paired evaluation design: $11$ of $40$ Open LLM Leaderboard v1 pairwise comparisons and $4$ of $9$ MMLU-Pro top-$10$ adjacent-rank pairs are unresolved at $(\alpha,1{-}\beta){=}(0.05,0.8)$. The MMLU-Pro count rises to $6/9$ under real subject-level clustering and stays at $5$--$6$ out of $9$ in $99.9\%$ of category-bootstrap resamples. We frame paired LLM evaluation as a hypothesis-testing problem, invert level-$\alpha$, power-$(1{-}\beta)$ tests, and report a per-pair resolution ratio $q{=}N/N^{\star}$ as the primary diagnostic. A sharp small-effect expansion with an explicit second-order constant shows that the widely-used unpaired Cohen-$h$-plus-$(1{-}\rho)$ shortcut deviates from the correct $N^{\star}$ by approximately a factor of two in the close-comparison regime, a deficit that three of five off-the-shelf calculators (Cohen 1988, G*Power, R \texttt{pwr}) silently inherit when the user post-multiplies their per-arm output by $(1{-}\rho)$. The unresolved-pair pattern remains under multiplicity correction and anytime-valid sequential testing.
\end{abstract}

\section{Introduction}
\label{sec:intro}

Modern LLM leaderboards rank models by percentage-point gaps on shared-prompt benchmarks. A leaderboard saying ``model $A$ scores $78.3\%$, model $B$ scores $77.5\%$'' converts a $0.8$-point gap into headlines and product decisions; but the assertion that $A$ is meaningfully better than $B$ is a statistical claim about the gap, not the gap itself. On the four-task Open LLM Leaderboard v1, $11$ of $40$ displayed pairwise rankings do not meet the paired-test resolution target at $(\alpha,1{-}\beta){=}(0.05,0.8)$; on the MMLU-Pro top-$10$, $4$ of $9$ adjacent-rank pairs remain unresolved at the actual benchmark size $N{=}12{,}032$. Concretely, the displayed gap between gemma-7B and Llama-3-8B on HellaSwag is $\hat\delta{=}{+}0.46$~pp at $n{=}10{,}042$: significant by asymptotic $\chi^2_1$ ($p{=}0.049$), not by the exact conditional binomial ($p{=}0.054$), with a paired-bootstrap $95\%$ CI on $\bar D$ containing zero. This is the kind of claim a resolution diagnostic surfaces. Whether a leaderboard has the resolution to support these claims depends on its size, on the within-prompt structure of the data, and on how many other adjacent gaps share the table.

The natural inference is a paired hypothesis test: for binary accuracy, McNemar's test on the discordant pairs with the corresponding required-$N$ \citep{mcnemar1947note,connor1987sample}; for graded scores, a paired-$t$ or paired bootstrap. The machinery is classical. What is missing is a \emph{resolution-reporting protocol} that says, given the size and structure of a benchmark, what gaps can be distinguished from sampling noise at conventional Type-I and Type-II error. Power calculators in common use focus on unpaired comparisons \citep{miller2024errorbars}, and there is no standard for what to publish alongside a headline gap.

We treat shared-prompt LLM benchmarks as paired hypothesis-testing problems and derive a resolution-reporting framework by inverting level-$\alpha$, power-$(1{-}\beta)$ tests. The framework yields three quantities: the minimum detectable effect (MDE) at the current $N$; the required paired sample size $N^{\star}$ at a target effect; and the resolution ratio $q=N/N^{\star}$.

The methodological building blocks are classical: Wald inversion, McNemar-Connor required-$N$, Bonferroni/Holm multiplicity, design-effect cluster correction, and anytime-valid e-processes. This paper contributes three things on top. First, a sharp characterisation of a specific misuse pattern (one new lemma with an explicit constant). Second, an empirical demonstration that the resulting diagnostic flags several displayed rankings as unresolved at the target $(\alpha, 1{-}\beta)$ resolution level. Third, a packaged reporting protocol that exposes each diagnostic entry as a one-line call. Concretely:

\noindent\emph{New theory.}
\begin{itemize}[leftmargin=*, itemsep=0pt, topsep=2pt]
\item A \textbf{sharp small-effect expansion} (\cref{lem:shortcut}) for the unpaired Cohen-$h$-plus-$(1{-}\rho)$ shortcut: the ratio $n_h/N^{\star}$ deviates from $\tfrac12$ by at most $C(p,\rho)\delta^2 + O(\delta^4)$ with an explicit second-order constant $C(p,\rho)$ that lets practitioners compute the admissible $\delta^{\star}$ at which the shortcut is reliable (\cref{cor:underest}). The leading-order factor of two is not new; the explicit constant and the uniform convergence on compact admissible sets are.
\end{itemize}

\noindent\emph{New empirical findings.}
\begin{itemize}[leftmargin=*, itemsep=0pt, topsep=2pt]
\item A \textbf{calculator-misuse characterisation} across five power-analysis tools at the worked example $(p_A,p_B,\rho){=}(0.65,0.60,0.30)$: three of five (Cohen's 1988 textbook formula, G*Power~3.1, R \texttt{pwr}) silently underestimate required-$N$ by a factor of two when their per-arm output is post-multiplied by $(1{-}\rho)$ (\cref{tab:misuse}). The misuse pattern is easily reproduced on every widely-used calculator we tried; whether it occurs in published LLM-evaluation work is a separate empirical question that we do not claim to settle here.
\item A \textbf{real-leaderboard reanalysis with prospective validation}: paired-McNemar required-$N$ is median $2.15{\times}$ smaller than the unpaired Gaussian formula of \citet{miller2024errorbars} on the same data (IQR $[1.60,2.75]$; \cref{fig:miller}), and on three real OLL v1 pairs the framework's $N^{\star}$ prescription delivers empirical McNemar power $0.80\pm 0.03$ over $M{=}1000$ bootstrap trials, with sub-prescription ($0.8\,N^{\star}$) and super-prescription ($1.2\,N^{\star}$) correctly under- and over-shooting (\cref{tab:prospective}). The diagnostic is calibrated on data, not just in asymptotic theory.
\item \textbf{Real subject-level cluster sensitivity}: using MMLU-Pro's $14$ subject categories as natural clusters, the unresolved adjacent-pair count rises from $4/9$ at IID to $6/9$ (\cref{tab:cluster}); two pairs flip from comfortably resolved to $N^{\star} > 3N$. A category-bootstrap CI on the cluster-corrected unresolved count and LOSO sensitivity (\cref{app:loso}) both support the flip.
\end{itemize}

\noindent\emph{Methodological integration.}
\begin{itemize}[leftmargin=*, itemsep=0pt, topsep=2pt]
\item Multiplicity, multi-arbiter, and anytime-valid sequential stress tests applied jointly to a single leaderboard family (\S\ref{sec:resolution}--\S\ref{sec:anytime}). Each component is classical; the joint pipeline and per-pair verdict table (\cref{tab:verdict}) are, to our knowledge, the first end-to-end resolution characterisation of a public LLM leaderboard.
\item A pip-installable package, \texttt{llm-power}, exposing the diagnostic for benchmark designers and leaderboard maintainers.
\end{itemize}

\noindent The empirical verdicts in \S\ref{sec:resolution} use \cref{eq:reqn_exact} directly and do not depend on the explicit constant in \cref{lem:shortcut}; the lemma's role is to sharply quantify the well-known factor-of-two pattern in a closed form practitioners can use to pre-screen their tools.

\paragraph{Paper roadmap.} \S\ref{sec:framework} defines the inversion and the resolution ratio $q$. \S\ref{sec:binary} instantiates for paired-binary accuracy and proves the shortcut lemma; \S\ref{sec:misuse} traces it through five real calculators. \S\ref{sec:calib} reports finite-sample calibration. \S\ref{sec:resolution} applies the diagnostic to OLL v1 and MMLU-Pro; \S\ref{sec:multiplicity}--\S\ref{sec:anytime} stress-test the verdict under multiplicity, real subject clustering, and anytime-valid sequential testing. The headline empirical finding is \cref{tab:verdict}; the headline theoretical finding is \cref{lem:shortcut}.

\section{Related work}
\label{sec:related}

\paragraph{Paired-binary tests.} McNemar's test \citep{mcnemar1947note} conditions on the discordant-pair count $b{+}c$. The large-sample required-$N$ formula is \citet{connor1987sample}; mid-$p$ \citep{liddell1983simplified} and continuity-corrected variants \citep{agresti2005simple} are the standard small-$b{+}c$ upgrades.

\paragraph{Power for NLP/LLM evaluation.} \citet{card2020little} argued NLP comparisons are routinely underpowered; \citet{dror2018hitchhiker} survey appropriate tests, including McNemar for paired binary outcomes. Methodologically closest is \citet{miller2024errorbars}, who gave a closed-form required-$N$ for the \emph{unpaired} Gaussian-accuracy case. The unpaired formula is appropriate for independent samples; our comparison quantifies the efficiency loss when the evaluation design is actually paired. On the same $40$ OLL v1 pairs we study (\S\ref{sec:resolution-v1}), the paired McNemar required-$N$ is a median $2.15\times$ smaller than Miller's unpaired formula at the empirical $\hat\rho$ (IQR $[1.60, 2.75]$, range $[1.33, 5.39]$); the empirical efficiency gain matches the textbook prediction $1/(1{-}\rho)$ for paired vs.\ unpaired Wald tests (\cref{fig:miller}, mean residual $-0.009$, max $0.062$).

\begin{figure}[t]
\centering
\includegraphics[width=\columnwidth]{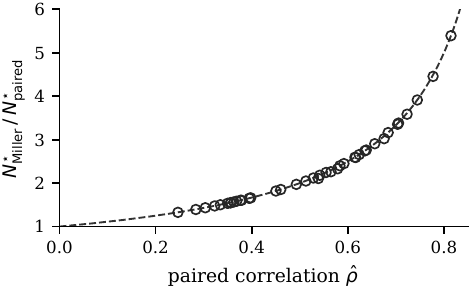}
\caption{Efficiency gain of paired McNemar over the unpaired Gaussian formula of \citet{miller2024errorbars} across the $40$ OLL v1 pairs. Open circles are empirical $N^{\star}_{\text{Miller}}/N^{\star}_{\text{paired}}$ at the pair's $\hat\rho$; dashed curve is the textbook prediction $1/(1{-}\rho)$ in the equal-marginal limit. Median empirical gain is $2.15\times$.}
\label{fig:miller}
\end{figure} \citet{madaan2024variance} measure benchmark variance across $13$ tasks. \citet{jo2025whatdoes} give a clustered bootstrap for ability-estimation precision; that is an estimator-variance statement, distinct from test power. \citet{polo2024tinybenchmarks} subsample benchmarks for efficient point-estimate accuracy, also distinct from the required-$N$ for paired hypothesis tests we address. Construct-validity critiques \citep{freiesleben2025epistemology,bean2025matters,alaa2025medical} are orthogonal: a benchmark can have high construct validity and still be too small for the test at hand.

To our knowledge, prior work treats these elements separately. The present paper integrates paired-difference variance with test inversion, leaderboard-scale multiplicity, and a per-pair resolution diagnostic on a real LLM benchmark, and packages the result as a reusable tool.

\section{Methodology}
\label{sec:framework}

\paragraph{Paired setup.} Two models $A, B$ are evaluated on the same $N$ prompts, which we treat as an i.i.d.\ sample from a target prompt superpopulation; without this, the gap on exactly these items is a fixed quantity and no hypothesis test is needed. Per-prompt scores $X_i^A, X_i^B$ may be binary, graded in $[0,1]$, or real-valued. Define the paired difference $D_i = X_i^A - X_i^B$ and the paired-mean estimator $\hat\delta = N^{-1}\sum_i D_i$.

\paragraph{Test and inversion.} Under a normal approximation,
\begin{equation}
\label{eq:wald}
T_N = \hat\delta/\widehat{\mathrm{SE}}(\hat\delta), \quad \widehat{\mathrm{SE}}(\hat\delta) = \sigma_D/\sqrt{N},
\end{equation}
where $\sigma_D^2 = \mathrm{Var}(D_i)$. For two-sided level-$\alpha$ tests, $H_0: \mathbb{E}[D_i] = 0$ is rejected when $|T_N| \geq z_{1-\alpha/2}$. Under a fixed alternative $\mathbb{E}[D_i] = \delta$,
\begin{equation}
\label{eq:power}
1-\beta = \Phi\!\left(\frac{|\delta|}{\sigma_D}\sqrt{N} - z_{1-\alpha/2}\right) + O(N^{-1/2}),
\end{equation}
where the $O(N^{-1/2})$ term is the standard finite-sample remainder of the normal approximation \citep{hall1992bootstrap}; \cref{sec:calib} characterises it empirically.\footnote{The exact two-sided power is $\pi_N(\delta) = \Phi(-z_{1-\alpha/2} - \mu) + 1 - \Phi(z_{1-\alpha/2} - \mu)$ with $\mu = |\delta|\sqrt{N}/\sigma_D$. The second term dominates at conventional target powers (the first contributes $<10^{-3}$ at $1{-}\beta = 0.8$); \cref{eq:nstar} inverts the dominant term.} Inverting at target power gives
\begin{align}
N^{\star}(\delta;\alpha,\beta) &= \!\left(\!\frac{(z_{1-\alpha/2}{+}z_{1-\beta})\,\sigma_D}{|\delta|}\!\right)^{\!2}\!, \label{eq:nstar}\\
\delta_{\mathrm{MDE}}(N;\alpha,\beta) &= \frac{(z_{1-\alpha/2}{+}z_{1-\beta})\,\sigma_D}{\sqrt{N}}. \label{eq:mde}
\end{align}
\begin{definition}[Resolution ratio]
\label{def:resratio}
For a paired leaderboard pair with observed gap $\hat\delta$ on $N$ shared prompts, the \emph{resolution ratio} is
$q \,:=\, N/N^{\star}(\hat\delta)$. The displayed gap is \emph{statistically resolvable} at $(\alpha,1{-}\beta)$ if $q \geq 1$.
\end{definition}
\noindent For a single pair, $q$ is a deterministic monotone transform of the squared Wald statistic: $q = T_N^2/(z_{1-\alpha/2}+z_{1-\beta})^2$, so $q\ge 1 \Leftrightarrow |T_N|\ge z_{1-\alpha/2}+z_{1-\beta} \approx 2.80$. Per-pair $q$ therefore carries no information beyond the (Wald-statistic) $p$-value; its value-add over a $p$-value is interpretive, plus the aggregation layer (multiplicity, clustering, anytime-validity) where the $q$ scale composes more naturally than $p$-values do.

\noindent $q < 1$ does not assert equality of the two models, nor does it overturn a fixed-$N$ $p$-value: it says the benchmark of this size does not have the target resolution for a gap of the displayed magnitude under the stated $(\alpha,1{-}\beta)$ operating point. The HellaSwag boundary pair in \S\ref{sec:intro} illustrates the distinction: it rejects asymptotically at $p{=}0.049$ yet has $q\approx \tfrac12$, so a nominally significant gap sits only halfway to the $(0.05, 0.8)$ resolution target. We use $q$ as the paper's load-bearing reporting quantity throughout.

\paragraph{Prospective vs.\ diagnostic.} \cref{eq:nstar} has two uses. Plugging in a \emph{pre-specified} target $\delta$ gives a prospective sample-size requirement for benchmark designers. Plugging in the \emph{observed} $\hat\delta$ gives a resolution diagnostic. We do \emph{not} compute power at the observed effect to argue that a particular ranking is true or false (the misuse criticised by \citet{hoenig2001abuse}); instead we compute $N^{\star}(\hat\delta)$ as a benchmark-design diagnostic: the required $N$ to detect a gap of the observed magnitude, irrespective of the significance verdict on the current sample.

\paragraph{Paired variance and multiplicity.} On shared prompts, $\sigma_D^2 = \mathrm{Var}(X^A) + \mathrm{Var}(X^B) - 2\,\mathrm{Cov}(X^A,X^B)$. State-of-the-art LLMs solve overlapping subsets of items, so $\mathrm{Cov}(X^A,X^B)$ is large and independent-proportion variances inflate MDE and $N^{\star}$. A $K$-model leaderboard simultaneously displays up to $\binom{K}{2}$ pairwise tests; replacing $\alpha$ with a Bonferroni-, Holm- \citep{holm1979simple} or BH-adjusted $\alpha'$ \citep{benjamini1995controlling} in \cref{eq:nstar} directly inflates $N^{\star}$. Family-level resolution is the relevant object whenever a leaderboard summarises many pairs at once.

\paragraph{Sequential / anytime-valid extension.} Public leaderboards update continuously, so the relevant $N$ is a stopping time chosen post-hoc. The same $\sigma_D^2$ plugs into a confidence sequence \citep{howard2021time,ramdas2023game} or a paired-Bernoulli mixture e-process, replacing $z_{1-\alpha/2}$ with a time-uniform boundary $u(n)$ satisfying $\Pr(\sup_n |T_n| \geq u(n)) \leq \alpha$ under $H_0$; \cref{sec:anytime} reports an ${\approx}2{\times}$ threshold inflation and one additional unresolved MMLU-Pro pair.

\section{Paired-binary instantiation and the shortcut lemma}
\label{sec:binary}

For binary accuracy the paired-difference variance is
\begin{equation}
\label{eq:vardiff}
\sigma_D^2 = p_A q_A + p_B q_B - 2\rho\sqrt{p_A q_A\,p_B q_B}, \;\, q_\cdot{=}1{-}p_\cdot,
\end{equation}
with $\rho$ the within-pair Bernoulli correlation. Substitution into \cref{eq:nstar} yields the McNemar-Connor required paired count
\begin{equation}
\label{eq:reqn_exact}
N^{\star} = \frac{(z_{1-\alpha/2}+z_{1-\beta})^2\,\sigma_D^2}{(p_A-p_B)^2},
\end{equation}
which agrees asymptotically with the $\chi^2_1$ McNemar test on the discordant counts $(b,c)\sim\mathrm{Binomial}(b{+}c,\tfrac12)$ under $H_0$ \citep{mcnemar1947note,connor1987sample}.

\paragraph{Admissible correlations.} Not every $\rho$ is achievable for given marginals $(p_A,p_B)$: the Hoeffding bound on Bernoulli correlation pins $\rho$ to an interval $[\rho_{\min},\rho_{\max}]$ with
\begin{align}
\rho_{\max} &= \sqrt{\tfrac{\min\{p_A(1{-}p_B),\,(1{-}p_A)p_B\}}{\max\{p_A(1{-}p_B),\,(1{-}p_A)p_B\}}},\\
\rho_{\min} &= -\sqrt{\tfrac{\min\{p_A p_B,\,(1{-}p_A)(1{-}p_B)\}}{\max\{p_A p_B,\,(1{-}p_A)(1{-}p_B)\}}}.
\end{align}
\cref{eq:vardiff} requires $\rho$ to lie in this admissible interval; we restrict all numerical claims accordingly.

\paragraph{The unpaired-to-paired shortcut.} A natural temptation, when only an unpaired Cohen-$h$ calculator is on hand, is to read off the per-arm $n_{\mathrm{unp}} = (z_{1-\alpha/2}+z_{1-\beta})^2/h^2$ \citep[Eq.~6.3.2]{cohen1988power} with $h = 2\arcsin\sqrt{p_A} - 2\arcsin\sqrt{p_B}$, then apply Cohen's generic $(1{-}\rho)$ paired adjustment \citep[Ch.~2]{cohen1988power}:
\begin{equation}
\label{eq:reqn_h}
n_h = (1-\rho)\,(z_{1-\alpha/2}+z_{1-\beta})^2/h^2.
\end{equation}
This is a natural shortcut when only an unpaired Cohen-$h$ calculator is to hand (see the comparison of \S\ref{sec:misuse}). \Cref{lem:shortcut} is a sharp small-effect characterisation of this failure mode: not a foundational theorem about paired tests, but a tight second-order expansion for a misuse that these tools make easy. The shortcut deviates from $N^{\star}$ by approximately a factor of two in the close-comparison limit; the deviation of $n_h/N^{\star}$ from $\tfrac12$ is $O(\delta^2)$, not $O(\delta)$, and the $\rho$-dependent term in $C$ vanishes at $p=\tfrac12$ (relevant since many benchmark accuracies cluster near $\tfrac12$). The technical contribution is the explicit constant $C(p,\rho)$ in \cref{eq:Cprho}.

\begin{lemma}[Sharp small-effect bound]
\label{lem:shortcut}
Fix $p \in (0,1)$ and $\delta$ small enough that $(p{+}\delta/2,\,p{-}\delta/2) \in (0,1)^2$, and $\rho$ in the Hoeffding-admissible region for this pair, with $\rho$ bounded away from $1$. Then for every sufficiently small $|\delta|$,
\begin{equation}
\label{eq:lem}
\left|\frac{n_h(p{+}\tfrac\delta2,\,p{-}\tfrac\delta2,\,\rho)}{N^{\star}(p{+}\tfrac\delta2,\,p{-}\tfrac\delta2,\,\rho)} - \frac{1}{2}\right| \;\leq\; C(p,\rho)\,\delta^2 \,+\, O(\delta^4),
\end{equation}
with explicit constant
\begin{equation}
\label{eq:Cprho}
C(p,\rho) = \frac{1}{2}\left|\frac{(1{+}\rho)(1{-}2p)^2}{16(1{-}\rho)\,p^2(1{-}p)^2} - \frac{1}{6\,p(1{-}p)}\right|,
\end{equation}
and convergence uniform on compact subsets of the admissible region. In particular $\lim_{\delta\to 0}\, n_h/N^{\star} = \tfrac12$, and at $p=\tfrac12$ the $(1{-}2p)^2$ term vanishes a fortiori, giving $C(\tfrac12, \rho) = 1/3$ independent of $\rho$.
\end{lemma}

\begin{corollary}[Underestimation in the close-comparison regime]
\label{cor:underest}
For any $\epsilon \in (0,\tfrac12)$ there exists $\delta^{\star}(p,\rho,\epsilon)$ such that whenever $|\delta| \le \delta^{\star}$, $n_h \le (\tfrac12 + \epsilon)\,N^{\star}$, i.e.\ the shortcut underestimates $N^{\star}$ by at least $(\tfrac12 - \epsilon)\,N^{\star}$. Concretely $\delta^{\star}(p,\rho,\epsilon) = \sqrt{\epsilon/C(p,\rho)}$ to leading order; e.g.\ $\delta^{\star}(0.65,0.3,0.05) \approx 0.43$.
\end{corollary}

\noindent \emph{Interpretation.} The shortcut $n_h$ is approximately one half of $N^{\star}$ at small $\delta$, with the deviation above or below $\tfrac12$ depending on $(p,\rho)$. The leading-order factor of two follows from $\mathrm{Var}(X^A-X^B)$ algebra and is not new; the explicit constant and the uniform convergence are what \cref{cor:underest} relies on to deliver a usable admissible $\delta^\star$.

\paragraph{When the shortcut matters operationally.} In the close-comparison regime, the shortcut produces approximately $N^\star/2$ with a deviation that is $O(\delta^2)$. What changes with $|\delta|$ is the operational consequence. A verdict flips (resolved vs.\ unresolved at benchmark size $N$) only when $N^{\star}$ and $n_h$ straddle $N$. For cross-tier comparisons where $\max(N^{\star}, n_h) \ll N$, both return ``resolved'' and the underestimate is benign. The verdict-changing regime is precisely the close-comparison regime that dominates leaderboard adjacencies: $17/40$ OLL v1 pairs have $|\hat\delta|\leq 5$~pp (\cref{tab:bucket}) and all $9$ MMLU-Pro top-$10$ adjacent pairs have $|\hat\delta|\leq 7$~pp (\cref{tab:discord_mmlupro}). The shortcut is material exactly where adjacent-rank claims rest, and benign where cross-tier comparisons would resolve either way.

\noindent \emph{Proof sketch (full proof in \cref{app:proof}).} The strategy is to Taylor-expand both $h^2$ and $\sigma_D^2$ around the midpoint $p$, exploiting the symmetry of $\arcsin\sqrt{\cdot}$ to kill the linear-in-$\delta$ terms. With $u = p(1{-}p)$, the result is $h^2 = (\delta^2/u)\cdot[1 + \delta^2\cdot H_2 + O(\delta^4)]$ where $H_2 = 1/(12u) + (1{-}2p)^2/(16u^2)$, and $\sigma_D^2 = 2u(1{-}\rho) + (\delta^2/4)\cdot[\rho(1{-}2p)^2/u - 2(1{-}\rho)] + O(\delta^4)$. Substituting and combining the two $O(\delta^2)$ corrections,
\(
n_h/N^{\star} = \tfrac12 - \tfrac{\delta^2}{2}\bigl[(1{+}\rho)(1{-}2p)^2/(16(1{-}\rho)u^2) - 1/(6u)\bigr] + O(\delta^4),
\)
from which \cref{eq:Cprho} follows. \qed

\paragraph{Numerical check.} Per-cell verification on a $(p,\rho,\delta)$ grid (\texttt{e6\_ratio\_heatmap.csv}; full details in \cref{app:proof}) matches $C(p,\rho)$ to four significant figures at small $\delta$. At $\delta{=}0.05$, $|n_h/N^{\star} - \tfrac12| \leq 0.0008$; at $\delta{=}0.20$, $\leq 0.014$. On the $40$ OLL v1 pairs (\S\ref{sec:resolution-v1}), the empirical ratio $n_h/N^{\star}$ has median $0.5002$, IQR $[0.4999, 0.5035]$, range $[0.487, 0.562]$; the $7$ close-comparison pairs ($|\hat\delta|\leq 2\%$) hit $\tfrac12$ to four decimals.

The structural cause is that $(1{-}\rho)$ is correctly applied to the single-arm variance $p(1{-}p)$, but the paired-difference variance \cref{eq:vardiff} carries an additional factor of two from $\mathrm{Var}(X^A){+}\mathrm{Var}(X^B)$. The fix is to use the paired-difference variance directly. We show next that this misuse pattern occurs in widely-used power tools.

\subsection{Calculator-misuse comparison}
\label{sec:misuse}

\begin{table}[t]
\centering\footnotesize
\caption{Calculator-misuse comparison at $(p_A,p_B,\rho)=(0.65,0.60,0.30)$, $(\alpha,1{-}\beta)=(0.05,0.8)$. ``Per-arm'' is the formula or calculator's native return value; ``$\times(1{-}\rho)$'' is the naive paired-sample-size readout. The correct $N^{\star}=1{,}028$ from \cref{eq:reqn_exact}. Three of the five formulas/tools return a per-arm $K/h^2$ that, when multiplied by $(1{-}\rho)$, gives the shortcut $n_h$: half the correct $N^{\star}$. \texttt{statsmodels} and \texttt{llm\_power} return values that do not suffer this misuse pattern. Tool versions and invocation commands in artifact under \texttt{appendix\_table1/}.}
\label{tab:misuse}
\setlength{\tabcolsep}{2pt}
\scriptsize
\begin{tabular}{@{}lcrr@{}}
\toprule
Calculator & convention & per-arm & $\times(1{-}\rho)$ \\
\midrule
Cohen 1988, Eq.~6.3.2                & $K/h^2$    & $736$    & $515$ \\
G*Power 3.1 (2-prop.\ $z$)           & $K/h^2$    & $736$    & $515$ \\
R \texttt{pwr::pwr.2p.test}          & $K/h^2$    & $736$    & $515$ \\
\texttt{statsmodels.NormalIndPower}  & $2K/h^2$   & $1{,}471$ & $1{,}030$ \\
\texttt{llm\_power} (paired)         & $\mathrm{Var}(\Delta)/\delta^2$ & -- & $1{,}028$ \\
\bottomrule
\end{tabular}
\end{table}

\cref{tab:misuse} traces the comparison. Three of the five tools (Cohen's textbook, G*Power 3.1, R's \texttt{pwr}) reproduce the shortcut's factor-of-two underestimate when $(1{-}\rho)$ is applied to a per-arm $K/h^2$ output, because the paired adjustment is being applied to a single-arm variance rather than to $\mathrm{Var}(\Delta)$. The remaining two recover the correct $N^{\star} \approx 1{,}028$: \texttt{statsmodels.NormalIndPower} uses a $2K/h^2$ per-arm convention that already accounts for both arms, and \texttt{llm\_power} computes the paired-difference variance directly. This is not a calculator-correctness issue: every tool returns a defensible quantity. It is a user-error pattern that the tools make easy, and \cref{lem:shortcut} is its quantitative consequence at the leaderboard adjacency regime.

\paragraph{Field practice.} The same pattern surfaces in the published literature. The unpaired Gaussian approximation of \citet{miller2024errorbars} is the prevalent closed-form required-$N$ treatment for LLM accuracy comparisons; paired-variance corrections via McNemar--Connor are rarely surfaced, even on data with shared prompts. The NLP power advocacy of \citet{card2020little} and the test-selection survey of \citet{dror2018hitchhiker} flag McNemar's availability but do not contrast paired and unpaired required-$N$ on the same data. \cref{lem:shortcut} and \cref{tab:misuse} together explain why this shortcut is quantitatively material at leaderboard adjacencies.

\section{Finite-sample calibration}
\label{sec:calib}

\Cref{eq:nstar} treats the rejection threshold as exact at any $N$, but the underlying normal approximation has a finite-$N$ remainder. We check empirically how big that remainder is. We calibrate five paired-binary test variants on simulated paired-Bernoulli populations: McNemar $\chi^2_1$, exact conditional binomial, mid-$p$, continuity-corrected McNemar, and the paired bootstrap of \cref{def:bpow}. The grid covers marginal accuracies $p\in\{0.5,0.7,0.9\}$, latent-Gaussian correlations $\rho_z\in\{0,0.4,0.8\}$ used to generate the Bernoulli pair, and sample size $n{=}500$. The per-cell $\delta$ under $H_1$ is tuned so the McNemar-Connor asymptotic target is exactly $0.80$ at each cell (Monte Carlo standard error ${\approx}0.6$~pp on Type-I, ${\approx}1.0$~pp on power). \Cref{tab:variants} reports cell-level Type-I deviations and empirical power; the bootstrap is the variant we recommend when no closed-form $\sigma_D$ is available.

\begin{definition}[Paired bootstrap test]
\label{def:bpow}
Given a paired score-matrix $\{(X^A_i,X^B_i)\}_{i=1}^N$, the two-sided percentile-bootstrap test of $H_0: \mathbb{E}[D_i]=0$ at level $\alpha$ resamples prompt indices with replacement, computes $\bar D^{(b)}$ on each resample, and rejects when $0 \notin \mathrm{CI}_{1-\alpha}$ from the percentile distribution of $\{\bar D^{(b)}\}_{b=1}^B$.\footnote{The percentile variant is chosen for simplicity; the studentized bootstrap \citep[Ch.~3]{hall1992bootstrap} is asymptotically more accurate. Empirically (\cref{tab:variants}), the percentile bootstrap's Type-I deviation on our $5$-cell grid is within $0.9$~pp of nominal, so the simpler variant suffices at the $N$ regime of interest.}
\end{definition}

\begin{table*}[t]
\caption{Empirical Type-I and power of five paired-binary test variants on a $5$-cell $(p,\rho,n)$ grid ($M{=}1500$ trials per cell; setup in \S\ref{sec:calib}).}
\label{tab:variants}
\centering\small
\setlength{\tabcolsep}{5pt}
\begin{tabular}{@{}llcccc@{}}
\toprule
Variant & Null & Data & $|\hat\alpha{-}\alpha|_{\max}$ (pp) & Power med.\ ($H_1$ at $0.80$) & max dev.\ from $0.80$ (pp) \\
\midrule
McNemar $\chi^2_1$        & marginal & binary        & $0.9$ & $0.79$ & $5.3$ \\
Exact conditional         & sharp    & binary        & $1.1$ & $0.76$ & $10.2$ \\
Mid-$p$                    & sharp    & binary        & $0.8$ & $0.79$ & $5.9$ \\
Continuity-corrected      & marginal & binary        & $1.1$ & $0.76$ & $10.6$ \\
Paired bootstrap (\cref{def:bpow})    & marginal & binary/graded & $0.9$ & $0.79$ & $5.2$ \\
\bottomrule
\end{tabular}
\\[2pt]
{\footnotesize Randomisation on discordant signs coincides with the exact conditional binomial in the binary case. The paired-$t$ / paired-bootstrap row applies to graded data only \citep{efron1993bootstrap,hall1992bootstrap}.}
\end{table*}

\paragraph{Empirical Type-I and power.} All five variants are calibrated within $1.1$~pp of $\alpha=0.05$. Under $H_1$ tuned to a $0.80$ asymptotic target, the asymptotic trio (McNemar $\chi^2_1$, mid-$p$, paired bootstrap) achieves median empirical power $0.79$, while the exact and continuity-corrected variants are ${\approx}3$~pp conservative; the remaining gap is the $O(N^{-1/2})$ remainder of the normal approximation in \cref{eq:power}. On synthetic Bernoulli and Beta$(4,2)$ graded marginals, the bootstrap tracks the parametric required-$N$ to within $4$--$6\%$ (\cref{app:synth}). \emph{Recommendation:} use \cref{eq:reqn_exact} with empirical $\hat\rho$ for binary accuracy; use the paired bootstrap of \cref{def:bpow} for graded metrics.

\section{Results}
\label{sec:resolution}

We apply the inversion to Open LLM Leaderboard v1 (\S\ref{sec:resolution-v1}) and the OLL v2 MMLU-Pro top-$10$ (\S\ref{sec:resolution-mmlupro}). All required-$N$ figures are at $(\alpha,1{-}\beta)=(0.05,0.8)$.

\subsection{OLL v1: 40 unique pairwise comparisons}
\label{sec:resolution-v1}

We pulled per-question $0/1$ scores for five $7$--$8$B-parameter open-weights models (Llama-3-8B$\pm$Instruct, Mistral-7B-Instruct-v0.2, Gemma-7B$\pm$it) on four tasks (ARC-Challenge, HellaSwag, Winogrande, GSM8K) from the EleutherAI \texttt{lm-evaluation-harness} dumps released via the OLL \texttt{details\_*} repos. Each task contributes $\binom{5}{2}=10$ unique pairwise comparisons (not adjacent-rank only), $40$ comparisons in total.

\Cref{fig:headline} plots $N^{\star}$ from \cref{eq:reqn_exact} at the observed $(\hat p_A,\hat p_B,\hat\rho)$ against the actual benchmark size; markers above $y=x$ are unresolved. The $|\delta|$-binned breakdown (\cref{tab:bucket}) shows the resolution boundary near $|\delta|\approx 5\%$: every pair with $|\delta|\leq 2\%$ is unresolved and every pair with $|\delta|>5\%$ is resolved; the $2$--$5\%$ band is mixed.

\begin{figure*}[t]
\centering
\includegraphics[width=\textwidth]{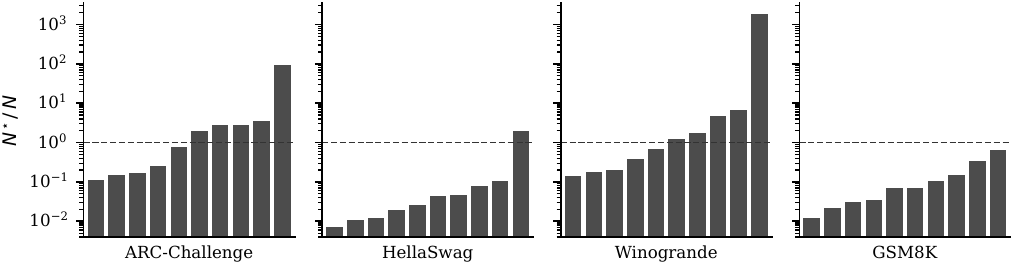}
\caption{Resolution diagnostic on $40$ OLL v1 pairwise comparisons, faceted by task. Each bar gives $r = N^{\star}/N = 1/q$ for one pair, sorted ascending within task; $N^{\star}$ at observed $(\hat p_A,\hat p_B,\hat\rho)$ via \cref{eq:reqn_exact}. Bars above the dashed line ($r > 1$) are unresolved at $(\alpha,1{-}\beta)=(0.05,0.8)$ ($11/40$ across all panels).}
\label{fig:headline}
\end{figure*}

\begin{table}[t]
\caption{Fraction of OLL v1 pairs with resolution ratio $q < 1$, binned by $|\delta|$. The two ratio columns are $r = N^{\star}/N = 1/q$ summarised across pairs in each bucket; $r > 1$ means unresolved.}
\label{tab:bucket}
\centering\footnotesize
\setlength{\tabcolsep}{4pt}
\begin{tabular}{lrrrr}
\toprule
$|\delta|$ bucket & pairs & unresolved & $r$ med.\ & $r$ worst \\
\midrule
$\leq 1\%$       &  $3$ & $3$ ($100\%$)  & $94$         & $1{,}892$ \\
$1\%$--$2\%$     &  $4$ & $4$ ($100\%$)  & $4.2$        & $6.8$ \\
$2\%$--$5\%$     & $10$ & $4$ ($40\%$)   & $0.75$       & $2.8$ \\
$5\%$--$15\%$    & $17$ & $0$ ($0\%$)    & $0.15$       & $0.65$ \\
$>\!15\%$        &  $6$ & $0$ ($0\%$)    & $0.03$       & $0.07$ \\
\midrule
all              & $40$ & $11$ ($28\%$)  & $0.16$       & $1{,}892$ \\
\bottomrule
\end{tabular}
\end{table}

\paragraph{Bootstrap uncertainty on $N^{\star}$.} Plug-in $N^{\star}$ inherits the sampling uncertainty of $(\hat p_A,\hat p_B,\hat\rho)$. We resample prompts with replacement ($B=500$), recompute, and report the $5$th--$95$th-percentile interval. Among the $7$ OLL v1 close pairs ($|\delta|\leq 2\%$), $4$ have a $5$th-percentile $N^{\star}$ exceeding $N$ (robustly unresolved); the remaining $3$ have intervals that span the diagonal because $\hat\delta$ is poorly estimated near zero.

\paragraph{Prospective design validation.} To verify that the framework's $N^{\star}$ prescription actually achieves the target power on real data, we picked three OLL v1 pairs spanning $|\delta| \in [6.3, 10.1]$~pp, computed $N^{\star}$ via \cref{eq:reqn_exact} at $(\hat p_A,\hat p_B,\hat\rho)$, bootstrap-subsampled to $N^{\star}$ prompts, and ran McNemar at $\alpha{=}0.05$ over $M{=}1000$ trials each. Empirical power at $N^{\star}$ lands at $0.796$--$0.827$, within $\pm 2.7$pp of the $0.80$ target; at $0.8\,N^{\star}$ and $1.2\,N^{\star}$ it correctly under- and over-shoots. The framework's prescription is calibrated on the data, not only in asymptotic theory. Per-pair table in \cref{app:prospective_multiarbiter}.

\paragraph{Multi-arbiter agreement on close pairs.} On every $|\delta|\leq 2\%$ pair, four arbiters (asymptotic McNemar $\chi^2_1$, exact two-sided conditional binomial, mid-$p$, and the paired bootstrap of \cref{def:bpow}) agree on six of seven close pairs (all unanimous fail-to-reject); only the leaderboard-displayed-significant HellaSwag pair flagged in \S\ref{sec:intro} produces a split verdict ($\chi^2_1\!:\!p{=}.049$ rejects; exact: $p{=}.054$ does not). Per-pair $p$-values in \cref{app:prospective_multiarbiter}.

\subsection{MMLU-Pro paired item-level tightening (OLL v2)}
\label{sec:resolution-mmlupro}

For the largest OLL v2 task with accessible per-item correctness, we pull the gated \texttt{open-llm-leaderboard/*\_details} parquets for the top-$10$ MMLU-Pro models ($N=12{,}032$ items) and compute $N^{\star}$ for the $9$ adjacent-rank pairs. The paired-Bernoulli formula (\cref{eq:reqn_exact}) and the discordance-form McNemar required-$N$ agree to within $1\%$ on every pair (\cref{fig:mcnemar_mmlupro}). Four of nine adjacent-rank pairs are unresolved at $N=12{,}032$ under the adjacent-rank multiplicity convention $m=9$. The all-pairs convention $m=\binom{10}{2}=45$ inflates $N^{\star}$ by ${\approx}2.14$, but every unresolved pair already has $q < 1$, so all four verdicts stand.

\Cref{lem:shortcut} holds tightly on this data: the shortcut $n_h$ is below $N^{\star}$ by the predicted factor of two on every pair (median ratio $0.500$, range $[0.496, 0.500]$, despite $\rho$ ranging over $[0.45, 0.99]$).

\paragraph{External replication.} \Cref{app:frontier} reports a closed-source frontier panel (4 models, MMLU-Pro $N{=}1{,}350$ all-four-scored): one adjacent pair (Llama-4-Maverick vs.\ DeepSeek-V3.2, gap $2.8$~pp) below resolution, shortcut/$N^{\star}$ in $[0.496, 0.500]$ replicating \cref{lem:shortcut} on a different model class (illustrative; deployments rotate).

\begin{figure}[t]
\centering
\includegraphics[width=\columnwidth]{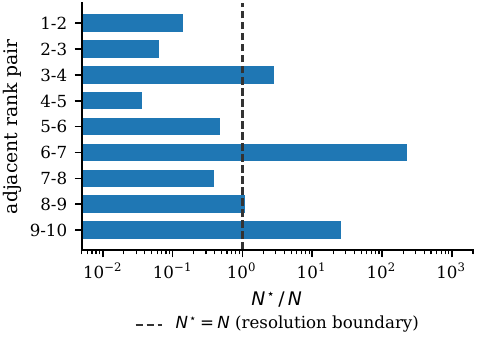}
\caption{MMLU-Pro top-$10$ adjacent pairs (OLL v2, $N=12{,}032$). Bars give $r = N^{\star}/N = 1/q$ for each adjacent-rank pair (log scale); the dashed line marks $r = 1$. Pairs whose bar extends past the dashed line ($r > 1$, unresolved) are unresolved at $(\alpha,1{-}\beta)=(0.05,0.8)$: four of nine.}
\label{fig:mcnemar_mmlupro}
\end{figure}

\paragraph{Sensitivity to $\hat\rho$ misspecification.} The diagnostic uses the empirical paired-correlation $\hat\rho$, which is itself estimated. Perturbing $\hat\rho$ by $\pm 0.10$ on every pair (clamped to the admissible interval) moves the OLL v1 unresolved count to $[9, 12]/40$ and the MMLU-Pro count to $[2, 4]/9$; the qualitative leaderboard message survives the perturbation. Three of $40$ OLL v1 pairs and $2$ of $9$ MMLU-Pro pairs flip verdict; in every case a boundary pair with $|\hat\delta|\le 4$~pp, including the HellaSwag pair from \S\ref{sec:intro} ($\hat\delta{=}0.46$~pp, $\hat\rho{=}0.81$) which becomes resolved at $\hat\rho{+}0.10{=}0.91$. Boundary-case sensitivity is expected.

\paragraph{Stress-test sequence.} The next three subsections stress the verdicts of \S\ref{sec:resolution} against family-level multiplicity (\S\ref{sec:multiplicity}), real subject-level clustering (\S\ref{sec:cluster}), and anytime-valid sequential testing (\S\ref{sec:anytime}). Each adjustment strictly tightens the verdict (\cref{tab:verdict}); the qualitative leaderboard message, that displayed adjacent gaps in this regime are often under-supported, survives all three.

\subsection{Leaderboard-scale multiplicity}
\label{sec:multiplicity}

We report multiplicity under two pre-declared families: \emph{adjacent-rank claims} (family size $K{-}1$) and \emph{all-pairs leaderboard claims} (family size $\binom{K}{2}$). Families are declared before observing $p$-values.

\paragraph{Bonferroni / Holm.} Running the $40$ OLL v1 pairs at nominal $\alpha=0.05$ inflates the family-wise error rate. Bonferroni control replaces $z_{1-\alpha/2}=1.960$ with $z_{1-\alpha/(2\cdot 40)}\approx 3.227$, multiplying $N^{\star}$ by ${\approx}2.11$. \v{S}id\'ak yields a near-identical $2.10\times$ inflation; Holm \citep{holm1979simple} is bounded above by Bonferroni at every position, so $2.11\times$ is an upper bound on the Holm-adjusted required-$N$. Applying $2.11\times$ to \cref{tab:bucket}, every $|\delta|\leq 2\%$ pair remains unresolved, the $2$--$5\%$ unresolved fraction rises from $4/10$ to $6/10$, and $16/17$ of the $5$--$15\%$ pairs stay resolved (the worst-case pair at $r{=}0.65$ flips since $0.65{\times}2.11{>}1$); the $>\!15\%$ bucket stays fully resolved.

\paragraph{Benjamini--Hochberg.} Controlling the false-discovery rate at $0.05$ via Benjamini--Hochberg \citep{benjamini1995controlling} is less conservative. On these data it effectively coincides with the $|\delta|\approx 5\%$ resolution boundary (close pairs produce $p$-values far from significance, large-gap pairs essentially at zero), so no $|\delta|>5\%$ rejection is retracted under FDR control.

\subsection{Cluster-aware sensitivity}
\label{sec:cluster}

Real benchmarks contain topic clusters, near-duplicates, and templated subdomains \citep{madaan2024variance,jo2025whatdoes} that induce intra-cluster correlation in the per-item paired difference $D_i$. The cluster-aware required-$N$ scales the IID estimate by the design effect $\mathrm{DE} = 1 + (\bar m - 1)\,\mathrm{ICC}(D)$, where $\bar m$ is mean cluster size and $\mathrm{ICC}(D)$ is the intra-cluster correlation of $D$. We compute $\mathrm{ICC}(D)$ empirically on MMLU-Pro using the dataset's $14$ subject categories \citep{madaan2024variance} as natural clusters ($\bar m \approx 859$ items per category); \cref{tab:cluster} reports per-pair ICC, DE, and cluster-adjusted $N^{\star}$ for the $9$ adjacent-rank top-$10$ pairs.

\paragraph{Real-data verdict.} Median empirical $\mathrm{ICC}(D) = 0.0010$ across the $9$ pairs; the distribution is heavy-tailed (range $[-3\!\times\!10^{-4},\, 3.6\!\times\!10^{-2}]$). Median DE $=1.88$ but two pairs (rank $4$ vs $5$, rank $5$ vs $6$) show large clustering ($\mathrm{DE} = 31.5$ and $6.7$): paired-difference homogeneity is high within categories for these pairs because the model gap is dominated by specific subject domains. Under real cluster correction, the unresolved count rises from $4/9$ at IID to $6/9$: rank $4$ vs $5$ flips from $N^{\star}_{\mathrm{IID}}=432$ (well-resolved) to $N^{\star}_{\mathrm{cluster}}=13{,}621$ ($>N$); rank $5$ vs $6$ flips from $5786$ to $39{,}009$.

\paragraph{Bootstrap stability of the verdict.} With only $K{=}14$ clusters, individual ICC point estimates carry nontrivial uncertainty. We address this by cluster-bootstrapping the verdict: resample the $14$ categories with replacement ($B{=}1000$, seed $42$) and recompute the unresolved-pair count. We hold $(\hat p_A,\hat p_B,\hat\rho)$ at the full-data estimates so the resulting CI isolates cluster-structure uncertainty. The three pairs driving the IID-to-cluster flip have $5$th-percentile ICC bounds that are all strictly positive (rank $4$ vs $5$ at $[0.021, 0.044]$, rank $5$ vs $6$ at $[0.003, 0.010]$, rank $6$ vs $7$ at $[0.004, 0.020]$), so the cluster signal that flips them is unlikely to be a $K{=}14$ artefact. Across the full bootstrap, the unresolved-pair count is $5/9$ in $45\%$ of resamples and $6/9$ in $55\%$; $\Pr(\text{unresolved}\!\geq\!5)=99.9\%$, and only $1$ of $1000$ resamples returns to the IID $4/9$ count. Full per-pair CIs are in \cref{app:icc_bootstrap}.

\paragraph{Sensitivity, not certification.} \S\ref{sec:cluster} should be read as a cluster-sensitivity analysis. Three robustness checks support the conclusion despite $K{=}14$: (i) the cluster-bootstrap above (verdict is at $5$--$6$ unresolved out of $9$ in $99.9\%$ of resamples), (ii) a leave-one-subject-out (LOSO) recomputation that drops each MMLU-Pro category in turn (\cref{app:loso}), which holds the unresolved count at $6/9$ on $11$ of $14$ drops and $5/9$ on the remaining $3$, and (iii) the cluster-definition sensitivity below. None of these collapses the result below $5/9$ except under random clusters (the null check), so the headline $4/9\to 6/9$ flip is not an artifact of either category sampling or any single high-ICC subject.

\paragraph{Cluster-definition sensitivity.} Subject categories are one of many possible clusterings. We rerun the \S\ref{sec:cluster} pipeline under three alternatives: random clusters of $K{=}14$ (null check); difficulty quartiles ($K{=}4$, binned by per-item mean accuracy across the top-$10$); and subject sub-clusters ($K{=}28$, each subject split in half by item parity). Random clusters give $\mathrm{ICC}{\approx}0$ and revert to the IID $4/9$ verdict (the null check). Difficulty quartiles produce a stronger cluster signal (median $\mathrm{ICC}{=}0.019$, max $0.19$) and $9/9$ unresolved; subject sub-clusters give $5/9$. The unresolved count thus spans $[5, 9]/9$ across non-null definitions, with the headline $6/9$ at the midpoint.

\begin{table}[t]
\caption{Real-data cluster sensitivity on MMLU-Pro top-$10$ adjacent-rank pairs, using the dataset's $14$ subject categories as clusters ($K=14$, $\bar m\approx 859$). $\mathrm{ICC}(D)$ is the empirical intra-cluster correlation of the paired difference $D_i$ (one-way ANOVA estimator); $\mathrm{DE}=1+(\bar m-1)\mathrm{ICC}^+$ with $\mathrm{ICC}^+ = \max(\mathrm{ICC}, 0)$. Bold $N^{\star}$ exceeds the actual $N{=}12{,}032$.}
\label{tab:cluster}
\centering\footnotesize
\setlength{\tabcolsep}{4pt}
\begin{tabular}{lrrrrr}
\toprule
Pair & $|\hat\delta|$ pp & $\hat\rho$ & ICC$(D)$ & DE & Cluster $N^{\star}$ \\
\midrule
1 vs 2 & 1.18 & 0.92 & 0.0004 & 1.37  & 2{,}327 \\
2 vs 3 & 1.73 & 0.93 & $-$0.0003 & 1.00 & 778 \\
3 vs 4 & 0.10 & 0.99 & 0.0002 & 1.19  & \textbf{40{,}660} \\
4 vs 5 & 6.61 & 0.46 & 0.036  & 31.5  & \textbf{13{,}621} \\
5 vs 6 & 1.88 & 0.45 & 0.0067 & 6.74  & \textbf{39{,}009} \\
6 vs 7 & 0.08 & 0.49 & 0.012  & 11.5  & \textbf{$\geq\!10^{7}$} \\
7 vs 8 & 0.91 & 0.90 & $-$0.0002 & 1.00 & 4{,}628 \\
8 vs 9 & 0.86 & 0.75 & 0.0010 & 1.88  & \textbf{24{,}632} \\
9 vs 10 & 0.22 & 0.58 & 0.0029 & 3.52  & \textbf{$\geq\!10^{6}$} \\
\bottomrule
\end{tabular}
\end{table}

\subsection{Anytime-valid leaderboard testing}
\label{sec:anytime}

Public leaderboards update continuously: each new model triggers up to $K{-}1$ new pairwise tests against the existing entries. A fixed-$n$ test loses Type-I control when stopping can be decided after seeing the data. We replace the fixed McNemar-Connor threshold $z_{1-\alpha/2}$ by an \emph{anytime-valid} threshold (valid simultaneously at every $n$) derived from a paired-Bernoulli mixture e-process \citep{howard2021time,ramdas2023game,vovk2021evalues}. Construction, mixture choice, and Type-I/stopping calibration are in \cref{app:eprocess_construction}.

\paragraph{Applied to MMLU-Pro adjacencies.} At $N{=}12{,}032$ the time-uniform threshold inflates $N^{\star}$ by $2.15\times$ vs.\ fixed-$n$. The numerical proximity to the $m{=}45$ Bonferroni inflation ($2.14\times$, \S\ref{sec:resolution-mmlupro}) and the $m{=}40$ figure ($2.11\times$, \S\ref{sec:multiplicity}) is coincidence: anytime-validity gives a time-uniform boundary; Bonferroni is a family-size correction.

Anytime-valid testing flips one extra verdict relative to fixed-$n$. Five of nine adjacent-rank pairs are unresolved under anytime-validity, vs.\ $4/9$ under fixed-$n$ and Bonferroni-$9$. The extra pair is rank $5$ vs.\ $6$ ($\hat\delta{=}1.88$pp, $\hat\rho{=}0.45$): the fixed-$n$ exact McNemar $p$-value is $5.8\!\times\!10^{-5}$, comfortably rejecting. The anytime-valid threshold, however, is stricter than a one-shot test at the realised discordance, so the same data fails to cross it. \cref{fig:eprocess} traces sample $\log e_n$ trajectories under $H_0$ and $H_1$.

\begin{figure}[t]
\centering
\includegraphics[width=\columnwidth]{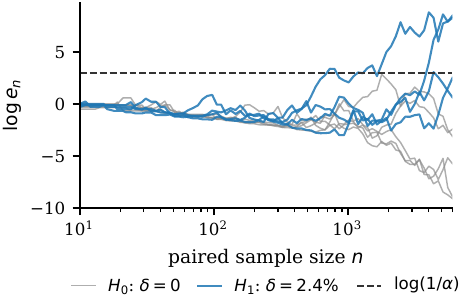}
\caption{Mixture e-process trajectories for paired Bernoulli under $H_0$ ($\delta{=}0$, grey) and $H_1$ ($\delta{=}2.4\%$, blue), calibrated to an ARC pair ($\hat\rho{=}0.64$). The horizontal dashed line is the rejection threshold $\log(1/\alpha)$ at $\alpha{=}0.05$. $H_0$ trajectories stay below; $H_1$ trajectories cross between $n{\sim}10^3$ and $n{\sim}5\!\times\!10^3$ (fixed-$n$ McNemar-Connor $N^{\star}{=}2362$ on this pair).}
\label{fig:eprocess}
\end{figure}

\paragraph{Verdict across procedures.} \cref{tab:verdict} pulls the unresolved counts together. On the larger OLL v1 family ($m{=}40$) the anytime-valid threshold-inflation at typical $N$ coincides with Bonferroni's; on the smaller MMLU-Pro family ($m{=}9$) anytime-validity dominates, flagging one additional pair as sequentially-not-resolvable. The columns of \cref{tab:verdict} answer subtly different questions and are not directly comparable as more-vs-less-correct: the anytime-valid column scores whether a verdict is resolvable \emph{under continuous monitoring of an infinitely-extending stream}, which is a stricter criterion than fixed-$n$ McNemar applied once.

\begin{table}[t]
\caption{Unresolved-pair count across testing paradigms.}
\label{tab:verdict}
\centering\footnotesize
\setlength{\tabcolsep}{4pt}
\begin{tabular}{@{}lcc@{}}
\toprule
Procedure & OLL v1 & MMLU-Pro \\
\midrule
Fixed-$n$              & $11/40$ & $4/9$ \\
Bonferroni / Holm      & $14/40$ & $4/9$ \\
Anytime-valid          & $14/40$ & $5/9$ \\
Real-ICC clustered     & n/a     & $6/9$ \\
\bottomrule
\end{tabular}
\end{table}

\section{Conclusion}
\label{sec:conclusion}

LLM leaderboard claims should report not only a gap and a $p$-value, but the benchmark resolution needed to make that gap detectable under the actual paired design.\footnote[2]{Code and raw responses: \url{https://github.com/akotawala10/llm-power}.} Across OLL v1 and MMLU-Pro, many displayed adjacent gaps fall below this resolution once pairing, multiplicity, clustering, or anytime-validity are made explicit. Clustering is verified by a category-bootstrap and LOSO (\cref{app:icc_bootstrap,app:loso}).

\paragraph{Scope.} The load-bearing claims (the shortcut lemma and the unresolved-pair verdicts) apply to the close-comparison regime where leaderboard adjacencies live. For cross-tier comparisons with large gaps, both paired and unpaired calculators return ``resolved'' at typical benchmark sizes, and the shortcut's factor-of-two underestimate is benign. The lemma is sharpest near $p{=}\tfrac12$, which is where many benchmark accuracies cluster.

\paragraph{Limitations.} Our empirics are restricted to binary accuracy, the displayed metric on the headline benchmarks; graded metrics and pairwise-preference leaderboards are handled only methodologically via \cref{def:bpow} (validated to $4$--$6\%$ on Beta$(4,2)$ marginals, \cref{app:synth}). The MMLU-Pro cluster analysis rests on $K{=}14$ subject categories; LOSO and a category bootstrap (\cref{app:loso,app:icc_bootstrap}) support the headline $4/9 \to 6/9$ flip, but a benchmark with finer-grained natural clusters would tighten the design-effect estimate. The closed-source frontier panel (\cref{app:frontier}) is a snapshot under rotating deployments and replicates \cref{lem:shortcut} only as illustration, not as a durable head-to-head comparison.

\paragraph{Future work.} Four directions extend the framework. (i) Generalise \cref{lem:shortcut} away from the equal-marginal midpoint and to multi-arm power. (ii) Derive a PRDS-aware required-$N$ for leaderboard families with overlapping models, building on \citet{benjamini2001control}. (iii) Plug in a fully clustered-prompt bootstrap with empirically-estimated ICC \citep{jo2025whatdoes} as the principled refinement of the design-effect estimate. (iv) Instantiate the resolution diagnostic on judge-scored and pairwise-preference (Chatbot-Arena-style) leaderboards via Bradley--Terry e-processes. Construct validity \citep{freiesleben2025epistemology,bean2025matters,alaa2025medical} composes with resolution: a benchmark can be statistically resolvable and still fail to measure what it claims to.

\bibliographystyle{icml2026}
\bibliography{refs_ht}

@article{miller2024errorbars,
  title         = {Adding Error Bars to Evals: A Statistical Approach to Language Model Evaluations},
  author        = {Miller, Evan},
  journal       = {arXiv preprint arXiv:2411.00640},
  year          = {2024},
  url           = {https://arxiv.org/abs/2411.00640}
}

@article{madaan2024variance,
  title         = {Quantifying Variance in Evaluation Benchmarks},
  author        = {Madaan, Lovish and Singh, Aaditya K. and Schaeffer, Rylan and Poulton, Andrew and Koyejo, Sanmi and Stenetorp, Pontus and Narang, Sharan and Hupkes, Dieuwke},
  journal       = {arXiv preprint arXiv:2406.10229},
  year          = {2024},
  url           = {https://arxiv.org/abs/2406.10229}
}

@article{jo2025whatdoes,
  title         = {What Does Your Benchmark Really Measure? A Framework for Robust Inference of {AI} Capabilities},
  author        = {Jo, Nathanael and Wilson, Ashia},
  journal       = {arXiv preprint arXiv:2509.19590},
  year          = {2025},
  url           = {https://arxiv.org/abs/2509.19590}
}

@article{freiesleben2025epistemology,
  title         = {The Benchmarking Epistemology: Construct Validity for Evaluating Machine Learning Models},
  author        = {Freiesleben, Timo and Zezulka, Sebastian},
  journal       = {arXiv preprint arXiv:2510.23191},
  year          = {2025},
  url           = {https://arxiv.org/abs/2510.23191}
}

@article{bean2025matters,
  title         = {Measuring what Matters: Construct Validity in Large Language Model Benchmarks},
  author        = {Bean, Andrew M. and Kearns, Ryan Othniel and Romanou, Angelika and Hafner, Franziska Sofia and Mayne, Harry and others},
  journal       = {arXiv preprint arXiv:2511.04703},
  year          = {2025},
  note          = {NeurIPS 2025 Datasets and Benchmarks Track},
  url           = {https://arxiv.org/abs/2511.04703}
}

@article{alaa2025medical,
  title         = {Position: Medical Large Language Model Benchmarks Should Prioritize Construct Validity},
  author        = {Alaa, Ahmed and Hartvigsen, Thomas and Golchini, Niloufar and Dutta, Shiladitya and Dean, Frances and Raji, Inioluwa Deborah and Zack, Travis},
  journal       = {arXiv preprint arXiv:2503.10694},
  year          = {2025},
  note          = {ICML 2025 (Position track)},
  url           = {https://arxiv.org/abs/2503.10694}
}

@book{cohen1988power,
  title         = {Statistical Power Analysis for the Behavioral Sciences},
  author        = {Cohen, Jacob},
  edition       = {2nd},
  publisher     = {Lawrence Erlbaum Associates},
  year          = {1988}
}

@book{efron1993bootstrap,
  title         = {An Introduction to the Bootstrap},
  author        = {Efron, Bradley and Tibshirani, Robert J.},
  publisher     = {Chapman \& Hall/CRC},
  year          = {1993}
}

@inproceedings{card2020little,
  title         = {With Little Power Comes Great Responsibility},
  author        = {Card, Dallas and Henderson, Peter and Khandelwal, Urvashi and Jia, Robin and Mahowald, Kyle and Jurafsky, Dan},
  booktitle     = {Proceedings of the 2020 Conference on Empirical Methods in Natural Language Processing (EMNLP)},
  pages         = {9263--9274},
  year          = {2020}
}

@inproceedings{dror2018hitchhiker,
  title         = {The Hitchhiker's Guide to Testing Statistical Significance in Natural Language Processing},
  author        = {Dror, Rotem and Baumer, Gili and Shlomov, Segev and Reichart, Roi},
  booktitle     = {Proceedings of the 56th Annual Meeting of the Association for Computational Linguistics (ACL)},
  pages         = {1383--1392},
  year          = {2018}
}

@article{connor1987sample,
  title         = {Sample Size for Testing Differences in Proportions for the Paired-Sample Design},
  author        = {Connor, Robert J.},
  journal       = {Biometrics},
  volume        = {43},
  number        = {1},
  pages         = {207--211},
  year          = {1987}
}

@article{mcnemar1947note,
  title         = {Note on the Sampling Error of the Difference Between Correlated Proportions or Percentages},
  author        = {McNemar, Quinn},
  journal       = {Psychometrika},
  volume        = {12},
  number        = {2},
  pages         = {153--157},
  year          = {1947}
}

@article{hoenig2001abuse,
  title         = {The Abuse of Power: The Pervasive Fallacy of Power Calculations for Data Analysis},
  author        = {Hoenig, John M. and Heisey, Dennis M.},
  journal       = {The American Statistician},
  volume        = {55},
  number        = {1},
  pages         = {19--24},
  year          = {2001}
}

@article{agresti2005simple,
  title         = {Simple Improved Confidence Intervals for Comparing Matched Proportions},
  author        = {Agresti, Alan and Min, Yongyi},
  journal       = {Statistics in Medicine},
  volume        = {24},
  number        = {5},
  pages         = {729--740},
  year          = {2005}
}

@article{liddell1983simplified,
  title         = {Simplified Exact Analysis of Case-Referent Studies: Matched Pairs; Dichotomous Exposure},
  author        = {Liddell, F. D. K.},
  journal       = {Journal of Epidemiology and Community Health},
  volume        = {37},
  number        = {1},
  pages         = {82--84},
  year          = {1983}
}

@article{benjamini1995controlling,
  title         = {Controlling the False Discovery Rate: A Practical and Powerful Approach to Multiple Testing},
  author        = {Benjamini, Yoav and Hochberg, Yosef},
  journal       = {Journal of the Royal Statistical Society: Series B (Methodological)},
  volume        = {57},
  number        = {1},
  pages         = {289--300},
  year          = {1995}
}

@book{hall1992bootstrap,
  title         = {The Bootstrap and Edgeworth Expansion},
  author        = {Hall, Peter},
  publisher     = {Springer},
  year          = {1992}
}

@article{holm1979simple,
  title         = {A Simple Sequentially Rejective Multiple Test Procedure},
  author        = {Holm, Sture},
  journal       = {Scandinavian Journal of Statistics},
  volume        = {6},
  number        = {2},
  pages         = {65--70},
  year          = {1979}
}

@article{benjamini2001control,
  title         = {The Control of the False Discovery Rate in Multiple Testing under Dependency},
  author        = {Benjamini, Yoav and Yekutieli, Daniel},
  journal       = {Annals of Statistics},
  volume        = {29},
  number        = {4},
  pages         = {1165--1188},
  year          = {2001}
}

@article{howard2021time,
  title         = {Time-uniform, nonparametric, nonasymptotic confidence sequences},
  author        = {Howard, Steven R. and Ramdas, Aaditya and McAuliffe, Jon and Sekhon, Jasjeet},
  journal       = {Annals of Statistics},
  volume        = {49},
  number        = {2},
  pages         = {1055--1080},
  year          = {2021}
}

@article{ramdas2023game,
  title         = {Game-theoretic statistics and safe anytime-valid inference},
  author        = {Ramdas, Aaditya and Gr{\"u}nwald, Peter and Vovk, Vladimir and Shafer, Glenn},
  journal       = {Statistical Science},
  volume        = {38},
  number        = {4},
  pages         = {576--601},
  year          = {2023}
}

@inproceedings{polo2024tinybenchmarks,
  title         = {tiny{B}enchmarks: Evaluating {LLM}s with Fewer Examples},
  author        = {Polo, Felipe Maia and Weber, Lucas and Choshen, Leshem and Sun, Yuekai and Xu, Gongjun and Yurochkin, Mikhail},
  booktitle     = {Proceedings of the 41st International Conference on Machine Learning (ICML)},
  year          = {2024}
}

@article{waudby2024betting,
  title         = {Estimating Means of Bounded Random Variables by Betting},
  author        = {Waudby-Smith, Ian and Ramdas, Aaditya},
  journal       = {Journal of the Royal Statistical Society: Series B (Statistical Methodology)},
  volume        = {86},
  number        = {1},
  pages         = {1--27},
  year          = {2024}
}

@article{grunwald2024safe,
  title         = {Safe Testing},
  author        = {Gr{\"u}nwald, Peter and de Heide, Rianne and Koolen, Wouter},
  journal       = {Journal of the Royal Statistical Society: Series B (Statistical Methodology)},
  year          = {2024},
  note          = {Read paper, with discussion}
}

@article{vovk2021evalues,
  title         = {E-values: Calibration, combination and applications},
  author        = {Vovk, Vladimir and Wang, Ruodu},
  journal       = {Annals of Statistics},
  volume        = {49},
  number        = {3},
  pages         = {1736--1754},
  year          = {2021}
}

\clearpage
\appendix
\onecolumn

\begin{center}
\noindent\rule{\linewidth}{0.6pt}\\[2pt]
{\large\bfseries Appendix}\\[2pt]
\noindent\rule{\linewidth}{0.6pt}\\[6pt]
\textit{Supplementary material to ``Resolution Diagnostics for Paired LLM Evaluation.''
Sections~A--J cover, respectively, the proof of \cref{lem:shortcut}, synthetic Bernoulli/graded validation, the mixture e-process construction and calibration, the illustrative closed-source replication, raw discordance and pair details, leave-one-subject-out cluster robustness, the cluster-bootstrap CIs on MMLU-Pro ICC, prospective validation and multi-arbiter agreement on OLL v1 close pairs, a reporting checklist for paired-leaderboard claims, and the \texttt{llm-power} API.}
\end{center}
\vskip 0.3in

\section{Proof of \cref{lem:shortcut}}
\label{app:proof}

Let $\phi(t)=\arcsin\sqrt{t}$, so $h(p_A,p_B)=2[\phi(p_A)-\phi(p_B)]$ with $\phi'(t)=1/[2\sqrt{t(1{-}t)}]$ and $u(t):=t(1{-}t)$. Under the midpoint parameterization $p_A = p + \delta/2$, $p_B = p - \delta/2$, expand symmetrically around $p$:
\begin{align}
\phi(p_A) - \phi(p_B) &= \delta\,\phi'(p) + \frac{\delta^3}{24}\phi'''(p) + O(\delta^5),
\end{align}
the even-order ($\delta^2$, $\delta^4$, $\ldots$) Taylor terms cancelling by symmetry of $\phi$ around $p$. Squaring,
\begin{equation}
h^2 = 4\delta^2 \phi'(p)^2\,\bigl[1 + \tfrac{\delta^2}{12}\,\phi'''(p)/\phi'(p) + O(\delta^4)\bigr].
\end{equation}
Direct computation gives $\phi'(p)^2 = 1/(4u)$ and $\phi'''(p)/\phi'(p) = 1/u + 3(1{-}2p)^2/(4u^2)$, so
\begin{equation}
\label{eq:h2-expansion}
h^2 = \frac{\delta^2}{u}\,\Bigl[1 + \delta^2\!\left(\tfrac{1}{12u} + \tfrac{(1{-}2p)^2}{16u^2}\right) + O(\delta^4)\Bigr].
\end{equation}
The $O(\delta^2)$ correction inside the bracket contributes at the same order as the $\sigma_D^2$ correction below, so it cannot be dropped.

For the paired-difference variance \cref{eq:vardiff}, expand symmetrically. The sum $u(p_A)+u(p_B) = 2u(p) + (\delta^2/4)u''(p) + O(\delta^4)$, and since $u''(t)=-2$,
\begin{equation}
u(p_A)+u(p_B) = 2u(p) - \tfrac{\delta^2}{2} + O(\delta^4).
\end{equation}
For the product, write $u(p_A) \approx u(p) + (\delta/2)u'(p) + (\delta^2/8)u''(p)$ and similarly for $u(p_B)$, then
\begin{align}
u(p_A)u(p_B) &= \!\bigl[u(p) - \tfrac{\delta^2}{4}\bigr]^2 - \!\bigl[\tfrac{\delta}{2}u'(p)\bigr]^2 + O(\delta^4) \nonumber\\
&= u(p)^2 - \tfrac{\delta^2}{2}u(p) - \tfrac{\delta^2}{4}u'(p)^2 + O(\delta^4),
\end{align}
so
\begin{equation}
\sqrt{u(p_A)u(p_B)} = u(p) - \delta^2\!\left(\tfrac{1}{4} + \tfrac{u'(p)^2}{8u(p)}\right) + O(\delta^4).
\end{equation}
With $u'(p)=1{-}2p$, substitution into \cref{eq:vardiff} gives
\begin{align}
\sigma_D^2 &= 2u(p)(1{-}\rho) \nonumber\\
&\quad + \tfrac{\delta^2}{4}\!\left[\rho\,\tfrac{(1{-}2p)^2}{u(p)} - 2(1{-}\rho)\right] + O(\delta^4).
\end{align}

The shortcut-to-paired ratio is then
\begin{align}
\frac{n_h}{N^{\star}} &= \frac{(1{-}\rho)\,\delta^2}{h^2\,\sigma_D^2}
= \frac{1}{2\,(1 + \delta^2 H_2)(1 + \delta^2 V_2)} \nonumber\\
&= \tfrac{1}{2}\bigl[1 - \delta^2(H_2 + V_2) + O(\delta^4)\bigr],
\end{align}
where the corrections are
\begin{align}
H_2 &= \tfrac{1}{12u} + \tfrac{(1{-}2p)^2}{16u^2},\\
V_2 &= \tfrac{\rho(1{-}2p)^2}{8u^2(1{-}\rho)} - \tfrac{1}{4u}.
\end{align}
Combining,
\begin{equation}
H_2 + V_2 = \tfrac{(1{+}\rho)(1{-}2p)^2}{16(1{-}\rho)u^2} - \tfrac{1}{6u}.
\end{equation}

Taking absolute values,
\begin{equation}
\left|\frac{n_h}{N^{\star}} - \frac{1}{2}\right| \leq \frac{1}{2}\!\left|\frac{(1{+}\rho)(1{-}2p)^2}{16(1{-}\rho)u^2} - \frac{1}{6u}\right|\,\delta^2 + O(\delta^4),
\end{equation}
giving $C(p,\rho)$ as in \cref{eq:Cprho}. Convergence is uniform on any compact $(p,\rho)$ set in which $u(p)\ge u_0>0$ and $1-\rho \ge \eta_0 > 0$, because $C(p,\rho)$ is then bounded above and the $O(\delta^4)$ remainder uniform. \cref{cor:underest} follows by inverting $C(p,\rho)\delta^2 \le \epsilon$. \qed

\paragraph{Numerical cross-check.} Direct per-cell evaluation on the heatmap grid of \texttt{e6\_ratio\_heatmap.csv} confirms the leading-order coefficient: the empirical $|n_h/N^{\star} - \tfrac12|/\delta^2$ matches the closed-form $C(p,\rho)$ to four significant figures at small $\delta$ on every cell ($p \in \{0.5, 0.65, 0.8\}$, $\rho \in [0, 0.7]$, $\delta \in [0.005, 0.20]$; median relative error $0.08\%$ at $\delta \le 0.05$). The full inequality $|n_h/N^{\star} - \tfrac12| \le C(p,\rho)\delta^2 + O(\delta^4)$ from \cref{lem:shortcut} holds with the explicit $O(\delta^4)$ remainder: at the upper end of the grid (e.g.\ $(p,\rho,\delta) = (0.8, 0.5, 0.20)$), the $O(\delta^4)$ correction adds up to ${\approx}17\%$ on top of $C\delta^2$, which is the expected behaviour of a small-$\delta$ Taylor bound. At $p=0.5$ the $(1{-}2p)^2$ factor vanishes a fortiori, giving the clean sanity check $C(0.5,\rho) = 1/(12u) = 1/3$ for every admissible $\rho$ (independent of $\rho$). Off-midpoint the constant grows with both $|1{-}2p|$ and $\rho$: empirically $C(0.65, 0.0) \approx 0.31$, $C(0.8, 0.5) \approx 0.80$, $C(0.65, 0.9) \approx 0.67$.

\section{Synthetic Bernoulli and graded validation}
\label{app:synth}

\Cref{fig:synth} reports bootstrap power against $n$ on synthetic paired Bernoulli data with $p_A=0.65$, $\delta\in\{0.02,0.04,0.08\}$, $\rho\approx 0.3$ via a latent Gaussian copula, $N=30{,}000$. The bootstrap crosses the $0.8$ target within $5\%$ of $N^{\star}$ from \cref{eq:reqn_exact} for all three $\delta$. \Cref{fig:graded} repeats the exercise on Beta$(4,2)$ marginals (graded stress test); the bootstrap tracks the paired-$t$ required-$N$ to within $4$--$6\%$.

\begin{figure}[h]
\centering
\includegraphics[width=0.7\columnwidth]{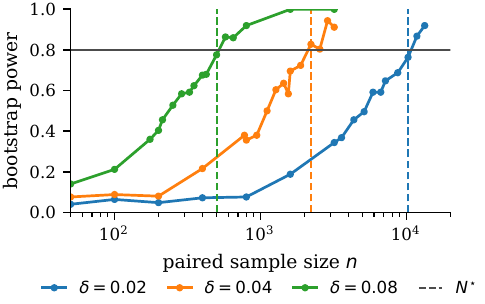}
\caption{Bootstrap power against paired sample size $n$ on synthetic paired Bernoulli data, three $\delta$. Dashed verticals mark $N^{\star}$ from \cref{eq:reqn_exact}.}
\label{fig:synth}
\end{figure}

\begin{figure}[h]
\centering
\includegraphics[width=0.7\columnwidth]{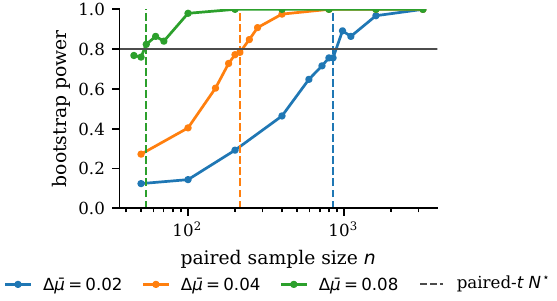}
\caption{Non-Bernoulli stress test: paired graded scores from Beta$(4,2)$ marginals. Bootstrap tracks the paired-$t$ required-$N$ within $\pm 6\%$.}
\label{fig:graded}
\end{figure}

\section{Mixture e-process: construction and calibration}
\label{app:eprocess_construction}

\paragraph{Construction.} Conditioned on $b{+}c$ discordant pairs, the sign sequence is i.i.d.\ Bernoulli$(\tfrac12)$ under $H_0$. With a discrete mixture prior $\nu$ over the alternative discordance probability $\theta \in (0, \tfrac12) \cup (\tfrac12, 1)$, the e-process is
\begin{equation}
\label{eq:eprocess}
e_n \;=\; \int \frac{\theta^{b_n}(1-\theta)^{c_n}}{(1/2)^{b_n + c_n}}\,\mathrm{d}\nu(\theta),
\end{equation}
By Ville's inequality (a martingale maximal inequality), $e_n$ has $\Pr(\sup_n e_n \geq 1/\alpha) \leq \alpha$ at any stopping time. The rejection rule ``reject the first $n$ at which $e_n \geq 1/\alpha$'' is therefore anytime-valid.

\paragraph{Mixture choice.} We use a discrete uniform mixture over $\theta \in \{0.01,\ldots,0.49,\,0.51,\ldots,0.99\}$ with equal weights, for two reasons. First, it is conjugate to the discordant-binomial likelihood, so the integral in \cref{eq:eprocess} is a closed-form sum. Second, it is the natural default in the betting-based formulation of safe testing \citep{waudby2024betting}, where the prior plays the role of a uniform prior over wagers. Alternatives that grow faster against a specific point alternative (e.g.\ GROW, \citealp{grunwald2024safe}) require knowing the alternative density and are sensitive to misspecification. Empirically, a Beta$(2,2)$ mixture and a discrete two-point mixture at $\theta \in \{0.4, 0.6\}$ both produce stopping times within ${\approx}8\%$ of the uniform on our calibration pairs.

\paragraph{Calibration.} We calibrate the mixture e-process on simulated paired Bernoulli calibrated to two ARC pairs ($\hat\delta=2.4\%, 7.8\%$; $\hat\rho=0.64, 0.54$; $M{=}600$ trials each). The empirical Type-I rates are $0.035$ and $0.043$ (Monte Carlo SE ${\approx}0.9$~pp, both within nominal $\alpha{=}0.05$); the $H_1$ rejection rate is $97$--$100\%$. The expected stopping time is $1.84$--$2.32{\times}$ the fixed-$n$ McNemar-Connor $N^{\star}$ on resolved pairs: the time-uniform cost of dropping the pre-specified-$n$ assumption.

\section{Closed-source frontier panel (illustrative)}
\label{app:frontier}

We replicate \cref{lem:shortcut} on a closed-source frontier panel as a robustness check on a different model class and evaluation route; we include it as illustrative rather than load-bearing because deployments rotate. Four pay-walled API models (GPT-5.5 and GPT-5.4 on Azure OpenAI; DeepSeek-V3.2 and Llama-4-Maverick-17B-128E-Instruct-FP8 on Azure AI Foundry) were evaluated on a $3{,}000$-item random subsample of MMLU-Pro on 2026-05-08 under a $5$-shot CoT prompt; dropping items where any model errored leaves $N{=}1{,}350$ all-four-scored items. Across the $3$ adjacent-rank pairs the shortcut/$N^{\star}$ ratio sits in $[0.496, 0.500]$ (median $0.497$), replicating Lemma 1's factor-of-two. The Llama-4-Maverick vs.\ DeepSeek-V3.2 pair (gap $2.8$~pp, $\rho=0.48$) is unresolved at $N{=}1{,}350$; the other two are resolved (\cref{fig:mcnemar_frontier}). Exact API identifiers, call timestamps, the verbatim $5$-shot CoT prompt, and the $3{,}000$-item index list (seed $42$ over the MMLU-Pro test split) are in the artifact under \texttt{appendix\_d/}.

\begin{figure}[!ht]
\centering
\includegraphics[width=0.55\columnwidth]{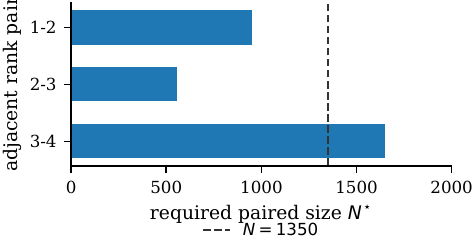}
\caption{Frontier panel on a $3{,}000$-item MMLU-Pro subsample ($N{=}1{,}350$). Bars give $N^{\star}$ for each adjacent-rank pair; the dashed line marks the actual $N$. Llama-4-Maverick vs.\ DeepSeek-V3.2 (rank 3 vs.\ 4) is unresolved at $(0.05, 0.8)$.}
\label{fig:mcnemar_frontier}
\end{figure}

\section{Raw discordance and pair details}
\label{app:discordance}

\Cref{tab:discord_oll,tab:discord_mmlupro} list the per-pair raw discordance counts $(b,c)$, marginals, correlation, and both $p$-value variants for the seven OLL v1 close pairs (\S\ref{sec:resolution-v1}) and the nine MMLU-Pro adjacent-rank pairs (\S\ref{sec:resolution-mmlupro}). These let a reader verify the headline HellaSwag $p_{\chi^2}{=}0.049$ vs.\ $p_{\mathrm{exact}}{=}0.054$ split (\S\ref{sec:intro}, row 4 of \cref{tab:discord_oll}) without re-running the code, and trace each rank pair to its concrete model identifiers.

\refstepcounter{table}\label{tab:discord_oll}%
\begin{center}\footnotesize\vspace{-1ex}
\textbf{Table~\thetable.} Raw discordance and pair details for the seven OLL v1 close pairs ($|\hat\delta|\le 2$pp). $b{=}n_{AB}$ and $c{=}n_{BA}$ are the discordant counts; $\hat\rho$ is the empirical Bernoulli correlation; $N^{\star}$ is the McNemar-Connor required-$N$ from \cref{eq:reqn_exact}. The HellaSwag row reconciles the \S\ref{sec:intro} example: $(b-c)/N = 46/10{,}042 = 0.46$\,pp. The $\approx 2.4{\times}10^{6}$ Winogrande Mistral-I/Llama-3-8B figure reflects $\hat\delta \approx 0$ ($b-c = 1$ out of $241$ discordant pairs) and should be read as ``far beyond resolution'' rather than a precise budget.\\[3pt]
\setlength{\tabcolsep}{4pt}
\begin{tabular}{@{}llllrrrrrrrr@{}}
\toprule
Task & Model A & Model B & $N$ & $\hat p_A$ & $\hat p_B$ & $b$ & $c$ & $\hat\rho$ & $p_{\chi^2}$ & $p_{\mathrm{exact}}$ & $N^{\star}$ \\
\midrule
ARC-C & gemma-7b           & Llama-3-8B-Instruct & 1{,}172 & 0.6109 & 0.6075 & 98  & 94  & 0.66 & 0.773 & 0.829 & 110{,}379 \\
ARC-C & Llama-3-8B-Instruct & Llama-3-8B          & 1{,}172 & 0.6075 & 0.5922 & 81  & 63  & 0.74 & 0.134 & 0.156 & 4{,}081 \\
ARC-C & gemma-7b           & Llama-3-8B          & 1{,}172 & 0.6109 & 0.5922 & 100 & 78  & 0.68 & 0.099 & 0.115 & 3{,}375 \\
HS    & gemma-7b           & Llama-3-8B          & 10{,}042 & 0.8247 & 0.8202 & 295 & 249 & 0.81 & \textbf{0.049} & \textbf{0.054} & 20{,}255 \\
Wino  & Mistral-7B-I-v0.2  & Llama-3-8B          & 1{,}267 & 0.7719 & 0.7711 & 121 & 120 & 0.46 & 0.949 & 1.000 & 2{,}396{,}624 \\
Wino  & gemma-7b           & Mistral-7B-I-v0.2   & 1{,}267 & 0.7845 & 0.7719 & 119 & 103 & 0.49 & 0.283 & 0.314 & 8{,}616 \\
Wino  & gemma-7b           & Llama-3-8B          & 1{,}267 & 0.7845 & 0.7711 & 98  & 81  & 0.59 & 0.204 & 0.232 & 6{,}152 \\
\bottomrule
\end{tabular}
\end{center}

\refstepcounter{table}\label{tab:discord_mmlupro}%
\begin{center}\footnotesize\vspace{-1ex}
\textbf{Table~\thetable.} Raw discordance and model identifiers for the nine MMLU-Pro top-$10$ adjacent-rank pairs ($N{=}12{,}032$). $p_{\chi^2}$ and $p_{\mathrm{exact}}$ are the asymptotic McNemar and exact conditional-binomial variants. $N^{\star}$ here is the IID McNemar--Connor value from \cref{eq:reqn_exact}; cluster-adjusted values are in \cref{tab:cluster}. Model names are the Open LLM Leaderboard v2 display names, lightly abbreviated for layout; full Hugging Face identifiers appear in the artifact at \texttt{experiments/a3\_mmlupro\_mcnemar.csv}.\\[3pt]
\setlength{\tabcolsep}{4pt}
\begin{tabular}{@{}lllrrrrr@{\hspace{0.7em}}r@{\hspace{0.7em}}r@{\hspace{0.7em}}r@{}}
\toprule
Pair & Model A & Model B & $\hat p_A$ & $\hat p_B$ & $b$ & $c$ & $\hat\rho$ & \multicolumn{1}{c}{$p_{\chi^2}$} & \multicolumn{1}{c}{$p_{\mathrm{exact}}$} & \multicolumn{1}{c}{$N^{\star}$} \\
\midrule
1 vs 2 & calme-3.2-78b      & calme-3.1-78b      & 0.7303 & 0.7185 & 253  & 111  & 0.92 & $9.9{\times}10^{-14}$ & $7.1{\times}10^{-14}$ & 1{,}697 \\
2 vs 3 & calme-3.1-78b      & CalmeRys-78B-Orpo  & 0.7185 & 0.7012 & 284  & 76   & 0.93 & ${<}10^{-15}$ & ${<}10^{-15}$ & 778 \\
3 vs 4 & CalmeRys-78B-Orpo  & calme-2.4-rys-78b  & 0.7012 & 0.7002 & 32   & 20   & 0.99 & 0.096 & 0.126 & 34{,}092 \\
4 vs 5 & calme-2.4-rys-78b  & Reflection-70B     & 0.7002 & 0.6341 & 1871 & 1076 & 0.46 & ${<}10^{-15}$ & ${<}10^{-15}$ & 433 \\
5 vs 6 & Reflection-70B     & Arcee-Blitz        & 0.6341 & 0.6153 & 1680 & 1454 & 0.45 & \mbox{$5.4\times10^{-5}$} & \mbox{$5.8\times10^{-5}$} & \mbox{5{,}787} \\
6 vs 7 & Arcee-Blitz        & Homer-Qwen2.5-72B  & 0.6153 & 0.6145 & 1449 & 1439 & 0.49 & \mbox{0.852} & \mbox{0.867} & \mbox{2{,}727{,}127} \\
7 vs 8 & Homer-Qwen2.5-72B  & ultiima-72B-v1.5   & 0.6145 & 0.6054 & 352  & 242  & 0.90 & \mbox{$6.4\times10^{-6}$} & \mbox{$7.3\times10^{-6}$} & \mbox{4{,}628} \\
8 vs 9 & ultiima-72B-v1.5   & Qwen2.5-72B        & 0.6054 & 0.5968 & 787  & 684  & 0.75 & \mbox{$7.2\times10^{-3}$} & \mbox{$7.8\times10^{-3}$} & \mbox{13{,}086} \\
9 vs 10 & Qwen2.5-72B        & QwentileSwap       & 0.5968 & 0.5945 & 1227 & 1200 & 0.58 & 0.584 & 0.598 & 314{,}370 \\
\bottomrule
\end{tabular}
\end{center}

\section{Leave-one-subject-out cluster robustness}
\label{app:loso}

\Cref{sec:cluster} reports a $4/9\to 6/9$ unresolved-count flip on MMLU-Pro top-$10$ adjacent pairs under real subject-level clustering ($K{=}14$ categories). To confirm the flip is not driven by any single category, we recompute the cluster-corrected unresolved count after dropping each MMLU-Pro category in turn and recomputing the per-pair ICC, design effect, and $N^{\star}_{\mathrm{cluster}}$ on the remaining items.

\Cref{tab:loso} lists the LOSO results: across the $14$ drops, the unresolved count stays at $6/9$ for $11$ drops and at $5/9$ for three drops (health, law, psychology). No drop collapses the result below $5/9$, so the cluster-induced flip is not an artifact of a single high-ICC category.

\refstepcounter{table}\label{tab:loso}%
\begin{center}\footnotesize\vspace{-1ex}
\textbf{Table~\thetable.} Leave-one-subject-out (LOSO) recomputation of the cluster-corrected unresolved-pair count on MMLU-Pro top-$10$ adjacent-rank pairs. Each row drops one of the $14$ subject categories and reruns the \S\ref{sec:cluster} pipeline on the remaining items.\\[3pt]
\setlength{\tabcolsep}{4pt}
\begin{tabular}{@{}>{\raggedright\arraybackslash}p{0.66\columnwidth}r@{}}
\toprule
Dropped category & Unresolved \\
\midrule
biology, business, chemistry, computer science, economics, engineering, history, math, other, philosophy, physics ($11$ drops) & $6/9$ \\
health, law, psychology ($3$ drops) & $5/9$ \\
\midrule
Base (no drop, \cref{tab:cluster}) & $6/9$ \\
\bottomrule
\end{tabular}
\end{center}

\section{Cluster-bootstrap CIs for MMLU-Pro ICC}
\label{app:icc_bootstrap}

\Cref{sec:cluster} reports a category-bootstrap stability check on the MMLU-Pro cluster-induced verdict flip. \Cref{tab:icc_bootstrap} lists the per-pair $5$--$95\%$ bootstrap intervals on $\mathrm{ICC}(D)$, design effect, and cluster-corrected $N^{\star}$. Resampling protocol: at each of $B{=}1000$ iterations (seed $42$), draw $K{=}14$ subject categories with replacement, recompute the one-way ANOVA $\mathrm{ICC}$ on the paired-difference series $D_i$ using the resampled cluster structure, derive $\mathrm{DE} = 1 + (\bar m - 1)\,\mathrm{ICC}^+$, and report $N^{\star}_{\mathrm{cluster}} = N^{\star}_{\mathrm{IID}}\cdot\mathrm{DE}$. The IID inputs $(\hat p_A,\hat p_B,\hat\rho,N^{\star}_{\mathrm{IID}})$ are held fixed at their full-data estimates so the CIs isolate cluster-structure uncertainty (full-bootstrap CIs would only be wider). The ``$\Pr(\text{unres.})$'' column reports the fraction of $B$ bootstraps in which $N^{\star}_{\mathrm{cluster}} > N{=}12{,}032$ for that pair.

\refstepcounter{table}\label{tab:icc_bootstrap}%
\begin{center}\scriptsize\vspace{-1ex}
\textbf{Table~\thetable.} Cluster-bootstrap CIs on MMLU-Pro top-$10$ adjacent-rank pairs ($B{=}1000$, $K{=}14$ categories resampled with replacement). Brackets are $5$--$95\%$ percentile bounds on the bootstrap distribution; ``pt'' is the full-data point estimate from \cref{tab:cluster}. ``$\Pr(\text{unr.})$'' is the fraction of bootstrap iterations for which the pair is unresolved at $N{=}12{,}032$. Three pairs (rank $4$ vs $5$, $5$ vs $6$, $6$ vs $7$) drive the cluster-induced verdict tightening and have ICC bounds (bold) well above zero. The bootstrap distribution of the count-of-unresolved-pairs (out of $9$) puts $44.9\%$ mass at $5$, $55.0\%$ at $6$, $0.1\%$ at $4$, and $0\%$ above $6$.\\[3pt]
\setlength{\tabcolsep}{2pt}
\begin{tabular}{lrrrr}
\toprule
Pair & ICC$(D)$ pt $[5,95\%]$ & DE pt $[5,95\%]$ & Cluster $N^{\star}$ pt $[5,95\%]$ & $\Pr(\text{unr.})$ \\
\midrule
$1$ vs $2$  & $0.0004$ $[-0.0003, 0.0012]$ & $1.37$ $[1.00, 1.95]$  & $2{,}327$ $[1{,}697,\,3{,}308]$ & $0.000$ \\
$2$ vs $3$  & $-0.0003$ $[-0.0008, 0.0001]$ & $1.00$ $[1.00, 1.12]$ & $778$ $[778,\,868]$ & $0.000$ \\
$3$ vs $4$  & $0.0002$ $[-0.0005, 0.0008]$ & $1.19$ $[1.00, 1.64]$ & $40{,}660$ $[34{,}092,\,55{,}765]$ & $1.000$ \\
$4$ vs $5$  & $\mathbf{0.036}$ $\mathbf{[0.021, 0.044]}$ & $31.5$ $[17.4, 40.9]$ & $13{,}621$ $[7{,}546,\,17{,}672]$ & $0.552$ \\
$5$ vs $6$  & $\mathbf{0.0067}$ $\mathbf{[0.0026, 0.0104]}$ & $6.74$ $[3.25, 9.78]$ & $39{,}009$ $[18{,}800,\,56{,}617]$ & $0.997$ \\
$6$ vs $7$  & $\mathbf{0.012}$ $\mathbf{[0.004, 0.020]}$ & $11.5$ $[4.12, 17.7]$ & $31.4{\rm M}$ $[11.2, 48.3]{\rm M}$ & $1.000$ \\
$7$ vs $8$  & $-0.0002$ $[-0.0007, 0.0002]$ & $1.00$ $[1.00, 1.15]$ & $4{,}628$ $[4{,}628,\,5{,}314]$ & $0.000$ \\
$8$ vs $9$  & $0.0010$ $[0.0002, 0.0016]$ & $1.88$ $[1.14, 2.36]$ & $24{,}632$ $[14{,}873,\,30{,}865]$ & $1.000$ \\
$9$ vs $10$ & $0.0029$ $[0.0015, 0.0037]$ & $3.52$ $[2.33, 4.08]$ & $1.11{\rm M}$ $[0.73, 1.28]{\rm M}$ & $1.000$ \\
\bottomrule
\end{tabular}
\end{center}

\section{Prospective design validation and multi-arbiter agreement}
\label{app:prospective_multiarbiter}

\Cref{tab:prospective} reports the prospective validation referenced in \S\ref{sec:resolution-v1}; \cref{tab:arbiters} reports the multi-arbiter agreement check on every $|\hat\delta| \leq 2$pp pair.

\refstepcounter{table}\label{tab:prospective}%
\begin{center}\footnotesize\vspace{-1ex}
\textbf{Table~\thetable.} Prospective validation: empirical McNemar power at framework-prescribed $N^{\star}$ on three OLL v1 pairs ($M{=}1000$ bootstrap trials per cell; Monte Carlo SE ${\approx}1.3$~pp at the $0.8$ target); columns probe sub-, on-, and super-prescription.\\[3pt]
\setlength{\tabcolsep}{4pt}
\begin{tabular}{@{}lcccccc@{}}
\toprule
 & & & & \multicolumn{3}{c}{Empirical power at} \\
\cmidrule(lr){5-7}
Pair & $|\delta|$ & $\hat\rho$ & $N^{\star}$ & $0.8\,N^{\star}$ & $N^{\star}$ & $1.2\,N^{\star}$ \\
\midrule
HS, Mistral-I / Llama-I  & $6.3$pp & $0.68$ & $193$ & $0.71$ & $\mathbf{0.83}$ & $0.89$ \\
ARC, Llama-3-8B / gemma  & $7.8$pp & $0.54$ & $294$ & $0.69$ & $\mathbf{0.80}$ & $0.88$ \\
HS, Llama-3-8B / gemma   & $10.1$pp & $0.57$ & $120$ & $0.72$ & $\mathbf{0.81}$ & $0.88$ \\
\bottomrule
\end{tabular}
\end{center}

\refstepcounter{table}\label{tab:arbiters}%
\begin{center}\footnotesize\vspace{-1ex}
\textbf{Table~\thetable.} Four-arbiter agreement on the seven OLL v1 close pairs ($|\hat\delta|\leq 2$pp). \texttt{R}/\texttt{F}: reject/fail-to-reject at $\alpha{=}0.05$. The HellaSwag boundary pair (bold) is the only split verdict.\\[3pt]
\setlength{\tabcolsep}{4pt}
\begin{tabular}{@{}lrrrcc@{}}
\toprule
Pair $\ (\hat\delta$pp$)$ & $\chi^2_1$ & exact & mid-$p$ & CI$\ni 0$ & verdict \\
\midrule
ARC, gem/L-I (0.34)        & .77  & .83  & .77  & yes & FFFF \\
ARC, L-I/L-8B (1.54)       & .13  & .16  & .13  & yes & FFFF \\
ARC, gem/L-8B (1.88)       & .10  & .12  & .10  & yes & FFFF \\
\textbf{HS, gem/L-8B (0.46)} & \textbf{.049} & \textbf{.054} & \textbf{.049} & \textbf{yes} & \textbf{RFRF} \\
Wi, Mi/L-8B (0.08)         & .95  & 1.00 & .95  & yes & FFFF \\
Wi, gem/Mi (1.26)          & .28  & .31  & .28  & yes & FFFF \\
Wi, gem/L-8B (1.34)        & .20  & .23  & .20  & yes & FFFF \\
\bottomrule
\end{tabular}
\end{center}

\section{Reporting checklist}
\label{sec:checklist}

\refstepcounter{table}\label{tab:checklist}%
\begin{center}\footnotesize\vspace{-1ex}
\textbf{Table~\thetable.} Reporting checklist for any paired LLM leaderboard claim. Each row is a direct consequence of an inversion in \S\ref{sec:framework} or its multiplicity adjustment in \S\ref{sec:multiplicity}.\\[3pt]
\setlength{\tabcolsep}{4pt}
\begin{tabular}{@{}ll>{\raggedright\arraybackslash}p{0.42\columnwidth}@{}}
\toprule
Quantity & Definition & Why \\
\midrule
$\hat\delta$               & observed gap                 & headline effect \\
Paired test                & McNemar / boot.\ / $t$       & matches metric \\
$N$                        & paired prompts                & budget \\
$\delta_{\mathrm{MDE}}$    & \cref{eq:mde}                 & current-$N$ resolution \\
$q = N/N^{\star}$           & \cref{def:resratio}           & gap supported when $q\ge 1$ \\
$N^{\star}$ CI             & $B{=}500$ bootstrap           & reveals whether $N^{\star}$ CI straddles $N$ \\
Multiplicity               & Bonf.\ / BH / none           & required for family-level claims only \\
Per-item raw               & $0/1$ matrix                  & third-party check \\
\bottomrule
\end{tabular}
\end{center}

\section{\texttt{llm-power} API}
\label{app:api}

\noindent\texttt{cohens\_h(p1, p2)} -- Cohen's $h$ for two proportions.\\
\texttt{paired\_bootstrap\_delta(scores\_a, scores\_b, ...)} -- prompt-bootstrap CI on $\hat\delta$.\\
\texttt{bootstrap\_power(...)} -- empirical power against a reference score-matrix.\\
\texttt{parametric\_required\_n\_proportions(p1, p2, paired, rho)} -- shortcut $n_h$ (Cohen convention).\\
\texttt{parametric\_required\_n\_paired\_binary(p1, p2, rho)} -- $N^{\star}$ from \cref{eq:reqn_exact}.\\
\texttt{required\_n\_mcnemar(n\_ab, n\_ba, n\_observed)} -- discordance-form McNemar required-$N$.\\
\texttt{parametric\_required\_n\_paired(mean\_diff, sd\_diff)} -- paired-$t$ required-$N$ for graded data.

\end{document}